\PassOptionsToPackage{table}{xcolor}

\documentclass[sigconf,nonacm]{acmart}
\setcopyright{none}   %
\usepackage{amsmath,bm,mathtools}
\usepackage{booktabs,multirow,array}
\usepackage{graphicx}
\usepackage{xurl}
\usepackage{enumitem}
\usepackage{xspace}
\usepackage{amssymb}
\usepackage[capitalise,noabbrev]{cleveref}   %

\definecolor{tabfirst}{rgb}{1, 0.7, 0.7}
\definecolor{tabsecond}{rgb}{1, 0.85, 0.7}
\definecolor{tabthird}{rgb}{1, 1, 0.7}

\newcommand{\MHeatmap}{\textsc{MHeatmap}\xspace}
\newcommand{\TW}{\textsc{Two-Walk}\xspace}
\newcommand{\OW}{\textsc{One-Walk}\xspace}
\newcommand{\Mosaic}{\textsc{Mosaic}\xspace}
\newcommand{\B}{\bm{B}}
\newcommand{\A}{\bm{A}}
\newcommand{\M}{\bm{M}}
\newcommand{\Lap}{\bm{L}}
\newcommand{\Atw}{\bm{A}_{\alpha}}
\newcommand{\Ltw}{\bm{L}_{\alpha}}
\newcommand{\one}{\bm{1}}
\newcommand{\R}{\mathbb{R}}

\theoremstyle{definition}
\newtheorem{definition}{Definition}
\theoremstyle{plain}
\newtheorem{theorem}{Theorem}
\newtheorem{proposition}{Proposition}
\newtheorem{lemma}{Lemma}
\theoremstyle{remark}

\graphicspath{{figures/}}

\title{MHeatmap: Two-Walk Laplacian Reordering and Mosaic Layout for Rectangular Matrix Visualization}

\author{Juntang Wang}
\affiliation{%
  \institution{Harvard University}
  \city{Cambridge}\country{USA}}
\email{juntangwang@g.harvard.edu}

\author{Sheng'en Li}
\affiliation{%
  \institution{Duke Kunshan University}
  \city{Kunshan}\country{China}}

\author{Liming Zhu}
\affiliation{%
  \institution{Duke Kunshan University}
  \city{Kunshan}\country{China}}
\email{liming.zhu@dukekunshan.edu.cn}

\author{Yihan Wang}
\affiliation{%
  \institution{Duke University}
  \city{Durham}\country{USA}}
\email{yihan.wang@dukekunshan.edu.cn}

\author{Siyuan Cao}
\affiliation{%
  \institution{Duke Kunshan University}
  \city{Kunshan}\country{China}}
\email{siyuan.cao@dukekunshan.edu.cn}

\author{Hao Wu}
\affiliation{%
  \institution{Duke Kunshan University}
  \city{Kunshan}\country{China}}
\email{hw412@duke.edu}

\author{Man Zhang}
\affiliation{%
  \institution{Westlake University}
  \city{Hangzhou}\country{China}}
\email{man.zhang@westlake.edu.cn}

\author{Shixin Xu}
\affiliation{%
  \institution{Duke Kunshan University}
  \city{Kunshan}\country{China}}
\email{shixin.xu@duke.edu}

\author{Dongmian Zou}
\authornote{Corresponding author.}
\affiliation{%
  \institution{Duke Kunshan University}
  \city{Kunshan}\country{China}}
\email{dongmian.zou@duke.edu}

\renewcommand{\shortauthors}{Juntang Wang et al.}

\begin{abstract}
Weighted rectangular matrices record relations between two types of objects, such as occupations and industries, documents and terms, origins and destinations, or annotated classes and computed clusters. Their heatmaps become difficult to read when arbitrary row and column orders scatter related objects, while equal size cells hide large differences in row and column totals. We present \textbf{\MHeatmap} (\textsc{\textbf{M}osaic \textbf{Heatmap}}), which addresses these problems in two stages. \TW spectral seriation first represents the matrix as a weighted bipartite graph. Ordinary bipartite ordering uses only observed row to column links. \TW also uses length-two walks, so rows that share columns and columns that share rows receive explicit similarity links. Sorting a Fiedler vector of the resulting graph Laplacian gives coordinated row and column orders. The \Mosaic layout then partitions a unit square so that row height is proportional to the row total and column width is proportional to the column total, while color represents the observed entry. We prove how the \TW energy combines direct alignment with shared-neighbor smoothness, derive an exact matrix vector product that permits sparse computation, and formalize the tiling and area semantics of \Mosaic. Across 300 synthetic matrices and seven real data sets, \TW obtains the lowest normalized 2-SUM in all 15 synthetic regimes and the highest Band@10\% in 14. The largest gains occur when shared neighborhoods express the latent relational structure.
\end{abstract}

\keywords{matrix visualization, matrix reordering, spectral reordering, mosaic heatmap, visual pattern recovery, bipartite graphs, marginal distribution}

\begin{document}

\maketitle

\section{Introduction}\label{sec:intro}
A weighted rectangular matrix $\B\in\R_{\ge 0}^{m\times n}$ records relations between two different entity types. Rows represent one type, columns represent the other, and $B_{ij}$ gives the nonnegative strength or frequency of their relation. Examples include education by occupation crosswalks, author by topic associations, document by term data, home by work commuting flows, and cell type by cluster contingency tables. Analysts inspect these matrices to identify correspondences, overlapping groups, split and merge relations, and mass that leaks outside expected blocks.

The input order is often alphabetical, administrative, or otherwise unrelated to relational structure. We use the term \emph{rectangular matrix seriation} for the task of permuting rows and columns so that related entities occupy nearby positions while both axes remain separately interpretable. A useful seriation can turn a scattered view into blocks or bands that are easier to inspect. Figure~\ref{fig:teaser} illustrates this effect on an OpenAlex author by topic matrix.

\begin{figure*}[t]
  \centering
  \includegraphics[width=\textwidth]{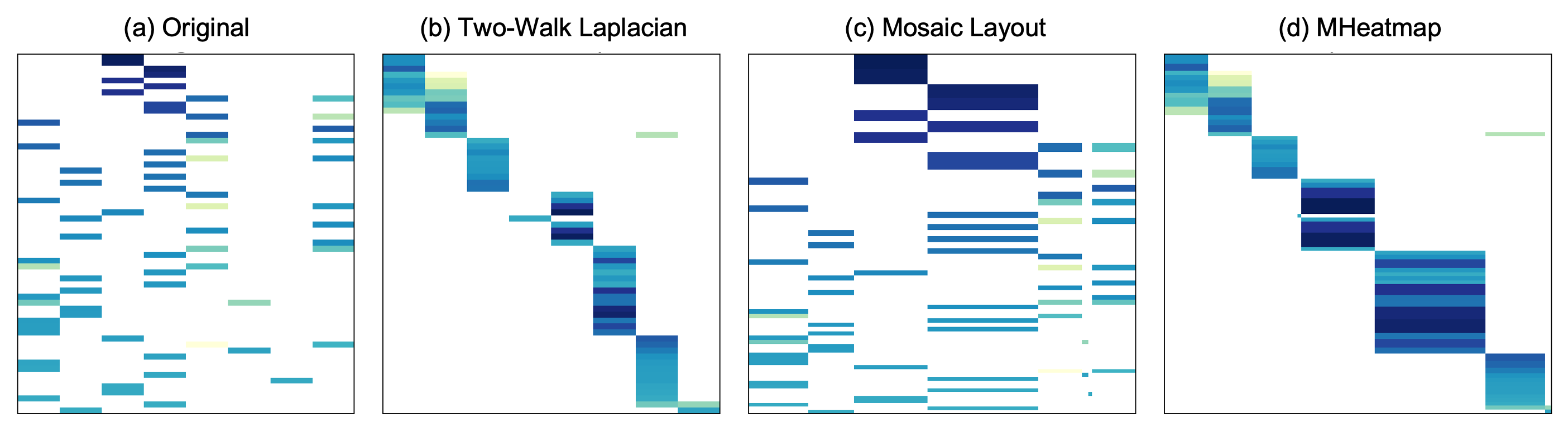}
  \caption{The \MHeatmap workflow on an OpenAlex author by topic matrix. Panel (a) uses the input order and equal size cells. Panel (b) applies \TW seriation and keeps equal size cells, which isolates the ordering effect. Panel (c) applies the \Mosaic layout to the input order, which isolates the geometry effect. Panel (d) combines \TW seriation and \Mosaic rendering.}
  \Description{Four views of one rectangular matrix. The original view is scattered. The Two-Walk view has diagonal blocks. The Mosaic view changes row heights and column widths. The full workflow combines both changes.}
  \label{fig:teaser}
\end{figure*}

A natural representation of $\B$ is a \emph{weighted bipartite graph}, which is a graph with one vertex set for rows, one vertex set for columns, and weighted edges only between the two sets. A standard spectral method places both sets in one symmetric adjacency matrix, forms the graph Laplacian $\Lap=\bm D-\A$, and sorts a \emph{Fiedler vector}. The graph Laplacian is the degree matrix minus the adjacency matrix. A Fiedler vector is an eigenvector associated with its smallest positive eigenvalue, and its entries provide a one dimensional coordinate for ordering. We call this direct construction \OW because its adjacency explicitly contains only direct row to column edges.

Direct edges do not explicitly state that two rows may connect to similar columns or that two columns may receive mass from similar rows. A \emph{length-two walk} follows two consecutive edges. In a bipartite graph, a length-two walk starts and ends on the same side, either row to column to row or column to row to column. These walks make shared neighborhoods explicit. We introduce \TW spectral seriation, which combines direct relations with the row similarity $\B\B^{\top}$ and the column similarity $\B^{\top}\B$ in one Laplacian. The resulting Fiedler vector gives coordinated row and column coordinates.

The second stage is the \Mosaic layout, a proportional heatmap that partitions a unit square into adjacent rectangles. A \emph{row marginal} is a row sum, and a \emph{column marginal} is a column sum. \Mosaic assigns row height in proportion to the row marginal and column width in proportion to the column marginal, while color encodes the observed entry. The rectangles cover the display as a complete tiling, which motivates the name Mosaic. For a contingency table, the area of a cell is the product of its row and column proportions. This is also the expected cell proportion when the two variables are independent. Geometry therefore shows the marginal baseline, while color shows the observed relation. This construction belongs to the historical family of mosaic and sieve displays, which Section~\ref{sec:related} discusses in detail.

The main contribution of this paper is the seriation method. The \Mosaic stage is part of the complete \MHeatmap workflow and is defined precisely so that its visual semantics are reproducible. Our contributions are:
\begin{itemize}[leftmargin=*,itemsep=1pt,topsep=2pt]
  \item We define \TW spectral seriation for nonnegative rectangular matrices. It combines direct cross-partition relations with same-side affinities induced by shared neighbors.
  \item We provide a Laplacian energy decomposition, a scale invariance result, a limit relation to \OW, and an exact sparse matrix vector product that avoids explicit Gram matrices.
  \item We evaluate the method on 300 controlled synthetic matrices and seven real relation matrices. The \OW comparison isolates the contribution of the length-two terms.
  \item We define the \Mosaic layout used by \MHeatmap and prove its tiling and area semantics. We also connect it to established mosaic and sieve displays.
\end{itemize}

\section{Related Work}
\label{sec:related}
\subsection{Matrix Seriation and Reordering}
Seriation arranges objects along an axis so that structurally related objects become nearby~\cite{liiv2010seriation}. Spectral seriation sorts a Fiedler vector and can be viewed as a relaxation of quadratic arrangement objectives~\cite{fiedler1973algebraic,atkins1998spectral,barnard1995spectral,fogel2015convex}. Other approaches include hierarchical clustering with optimal leaf ordering~\cite{barjoseph2001optimal}, bandwidth and profile reduction for sparse matrices~\cite{reid2006reducing}, and alternating procedures motivated by weighted bipartite layouts~\cite{cakiroglu2009weightedcrossing}. Recent visualization work evaluates whether an ordering reveals blocks, bands, and related motifs rather than judging only one numerical objective~\cite{behrisch2017magnostics,vanbeusekom2022reorder,zhu2025reorderbench}. Most classical formulations assume a square similarity matrix or order the two axes independently. We study coordinated row and column orders for a rectangular relation matrix.

\subsection{Bipartite Spectral Methods and Shared Neighborhoods}
Spectral co-clustering represents a rectangular matrix as a bipartite graph and uses coupled row and column coordinates to obtain discrete groups~\cite{dhillon2001coclustering}. Correspondence analysis and singular value methods also provide paired coordinates~\cite{greenacre2017correspondence}. Our task is seriation, which returns continuous orders rather than only clusters. The plain bipartite adjacency contains direct cross-partition edges but no explicit same-side edges.

A shared-neighbor relation is a similarity induced by common adjacent vertices. Such relations are standard in graph similarity and link analysis~\cite{blondel2004similarity,libennowell2007link}. For a bipartite adjacency $\M$, the square $\M^2$ records length-two walks and is block diagonal. It therefore captures row similarity and column similarity but separates the two vertex sets. \TW uses $\M^2+\alpha\M$ so that shared neighborhoods and direct relations contribute to one spectral coordinate. The contribution is this joint seriation formulation and its analysis, not the general observation that graph walks encode proximity.

\subsection{Mosaic, Sieve, and Product Plots}
The term mosaic has an established history in statistical graphics. Hartigan and Kleiner used it for a collection of rectangles whose areas represent contingency table counts~\cite{hartigan1981mosaics}. Friendly developed classical mosaic displays for multiway tables~\cite{friendly1994mosaic}. A related parquet or sieve diagram sets row heights and column widths from the corresponding marginals, so each rectangle has area proportional to the expected count under independence~\cite{riedwyl1994parquet,meyer2006strucplot}. Product plots provide a general framework in which area is the product of width and height and include mosaic, treemap, and equal area displays~\cite{wickham2011product}.

This lineage explains the name \Mosaic in \MHeatmap. Its geometry follows the marginal partition of a sieve diagram, while its color channel follows a heatmap by encoding the observed entry. The result is related to, but not identical with, a classical mosaic display. The general idea of a proportional rectangular tiling is established prior work. Our method-specific contribution is the precise combination of marginal geometry and heatmap color inside the \MHeatmap workflow. Marginal bars provide another direct way to show row and column totals. Equal area heatmaps, marginal bars, and \Mosaic support different reading tasks. Since area judgments are less precise than aligned position and length judgments~\cite{cleveland1984graphical}, our claims for \Mosaic concern its encoding semantics and qualitative use, not universal perceptual superiority.

\section{Method}
\label{sec:method}
\subsection{Preliminaries and One-Walk Seriation}
Let $\B\in\R_{\ge0}^{m\times n}$ be nonzero.

\begin{definition}[Active matrix]
The active matrix of $\B$ is obtained by deleting every row and column whose marginal is zero. These deleted indices have no incident graph edge and therefore no informative spectral coordinate.
\end{definition}

We state the spectral results for an active matrix whose bipartite support graph is connected. The \emph{support graph} keeps an edge exactly where the corresponding matrix entry is positive. Section~\ref{sec:implementation} gives a deterministic extension to disconnected support.

For any symmetric nonnegative adjacency $\A$, let $\bm D=\operatorname{diag}(\A\one)$ and $\Lap=\bm D-\A$. When the graph is connected, a unit Fiedler vector $\bm f$ satisfies the variational characterization
\begin{equation}
\bm f\in\arg\min_{\bm z^{\top}\one=0,\ \|\bm z\|_2=1}\bm z^{\top}\Lap\bm z.
\label{eq:rayleigh}
\end{equation}
Thus strongly connected vertices are encouraged to receive nearby coordinates.

\begin{definition}[Rectangular matrix seriation]
A rectangular seriation of $\B$ is a pair of permutations $(\pi_R,\pi_C)$ of its row and column indices. The permutations are coordinated when both are derived from one model of the row and column relations.
\end{definition}

Our implementation normalizes the active matrix by its largest entry. We therefore define
\begin{equation}
\bar{\B}=\frac{\B}{s(\B)},\qquad s(\B)=\max_{i,j}B_{ij}>0.
\label{eq:scale}
\end{equation}
This normalization is necessary because the two-walk term is quadratic in $\B$ and the direct term is linear.

\begin{definition}[Bipartite lift and One-Walk seriation]
The bipartite lift of $\bar{\B}$ is
\begin{equation}
\M=\begin{bmatrix}
\bm 0 & \bar{\B}\\
\bar{\B}^{\top} & \bm 0
\end{bmatrix}.
\label{eq:lift}
\end{equation}
Let $\bm L_{\mathrm{ow}}$ be the Laplacian of $\M$. Split a Fiedler vector as $[\bm x^{\top},\bm y^{\top}]^{\top}$, where $\bm x\in\R^m$ and $\bm y\in\R^n$. Sorting $\bm x$ and $\bm y$ separately defines \OW row and column orders.
\end{definition}

\begin{proposition}[Scale invariance]
For any scalar $\gamma>0$, replacing $\B$ by $\gamma\B$ leaves $\bar{\B}$, the \OW operator, and every \TW operator with the same $\alpha$ unchanged.
\end{proposition}
\begin{proof}
The maximum entry scales as $s(\gamma\B)=\gamma s(\B)$. Hence $(\gamma\B)/s(\gamma\B)=\B/s(\B)$.
\end{proof}

\subsection{Two-Walk Spectral Seriation}
A length-two walk in the bipartite lift returns to its starting partition. Squaring Eq.~\eqref{eq:lift} gives
\begin{equation}
\M^2=\begin{bmatrix}
\bar{\B}\bar{\B}^{\top} & \bm 0\\
\bm 0 & \bar{\B}^{\top}\bar{\B}
\end{bmatrix}.
\label{eq:m2}
\end{equation}
The first diagonal block measures row similarity through shared columns. The second measures column similarity through shared rows. Since $\M^2$ alone disconnects the two partitions, direct edges are retained.

\begin{definition}[Two-Walk spectral seriation]
For $\alpha>0$, define
\begin{equation}
\Atw=\M^2+\alpha\M
=\begin{bmatrix}
\bar{\B}\bar{\B}^{\top} & \alpha\bar{\B}\\
\alpha\bar{\B}^{\top} & \bar{\B}^{\top}\bar{\B}
\end{bmatrix}.
\label{eq:tw}
\end{equation}
Let $\Ltw$ be its graph Laplacian. Split a Fiedler vector of $\Ltw$ into row coordinates $\bm x$ and column coordinates $\bm y$. Sorting the two coordinate vectors separately defines the \TW seriation.
\end{definition}

\begin{proposition}[Connectivity]
\label{prop:connectivity}
If the bipartite support graph of the active matrix is connected and $\alpha>0$, then the graph with adjacency $\Atw$ is connected. Consequently, the nullspace of $\Ltw$ is $\operatorname{span}\{\one\}$, and its Fiedler eigenvalue is the smallest positive eigenvalue.
\end{proposition}
\begin{proof}
Every positive edge of the bipartite lift $\M$ remains a positive edge of $\Atw$ with weight multiplied by $\alpha$. Hence the graph of $\Atw$ contains the connected support graph of $\M$ as a spanning subgraph. The Laplacian of a connected undirected graph has a one dimensional nullspace spanned by $\one$.
\end{proof}

The diagonal entries of the two Gram matrices are self loops. They cancel in the combinatorial Laplacian and therefore do not alter the energy below.

\begin{theorem}[Two-Walk energy decomposition]
\label{thm:energy}
For $\bm z=[\bm x^{\top},\bm y^{\top}]^{\top}$,
\begin{equation}
\begin{aligned}
\bm z^{\top}\Ltw\bm z
={}&\frac{1}{2}\sum_{i,k}(\bar{\B}\bar{\B}^{\top})_{ik}(x_i-x_k)^2 \\
&+\frac{1}{2}\sum_{j,\ell}(\bar{\B}^{\top}\bar{\B})_{j\ell}(y_j-y_{\ell})^2 \\
&+\alpha\sum_{i,j}\bar B_{ij}(x_i-y_j)^2.
\end{aligned}
\label{eq:energy}
\end{equation}
\end{theorem}
\begin{proof}
For a symmetric adjacency $\A$, the identity $\bm z^{\top}(\bm D-\A)\bm z=\frac12\sum_{p,q}A_{pq}(z_p-z_q)^2$ follows by expansion. Applying it to the four blocks of Eq.~\eqref{eq:tw} gives the two same-side sums. The two symmetric cross blocks together give the final sum.
\end{proof}

The first term makes rows with similar column neighborhoods smooth in the spectral coordinate. The second does the same for columns. The third aligns directly related rows and columns. Thus $\alpha$ controls the relative emphasis on direct alignment and shared-neighbor smoothness.

\begin{proposition}[Relation to One-Walk]
\label{prop:limit}
Let $\Lap(\M)$ and $\Lap(\M^2)$ denote the Laplacians of $\M$ and $\M^2$. Then
\begin{equation}
\Ltw=\Lap(\M^2)+\alpha\Lap(\M),
\qquad
\alpha^{-1}\Ltw\longrightarrow\Lap(\M)
\end{equation}
as $\alpha\to\infty$. If the Fiedler eigenvalue of $\Lap(\M)$ is simple, the normalized \TW Fiedler vector converges to the \OW Fiedler vector up to sign. If its row and column coordinates are distinct, the induced permutations also agree for sufficiently large $\alpha$.
\end{proposition}
\begin{proof}
Laplacian construction is linear in the adjacency, which gives the first identity. Dividing by $\alpha$ yields $\Lap(\M)+\alpha^{-1}\Lap(\M^2)$. The second term converges to zero in operator norm. Standard perturbation theory for a simple eigenvalue gives convergence of the eigenvector up to sign. Distinct coordinates preserve their sorted order under sufficiently small perturbations.
\end{proof}

\subsection{Exact Matrix Vector Product}
\label{sec:implementation}
Explicitly forming $\bar{\B}\bar{\B}^{\top}$ and $\bar{\B}^{\top}\bar{\B}$ can destroy sparsity. The following identity permits a sparse eigensolver without constructing either matrix.

\begin{lemma}[Matrix-free Two-Walk product]
\label{lem:matvec}
For $\bm z=[\bm x^{\top},\bm y^{\top}]^{\top}$, define $\bm u=\bar{\B}^{\top}\bm x$ and $\bm v=\bar{\B}\bm y$. Then
\begin{equation}
\Atw\bm z=
\begin{bmatrix}
\bar{\B}\bm u+\alpha\bm v\\
\bar{\B}^{\top}\bm v+\alpha\bm u
\end{bmatrix}.
\label{eq:matvec}
\end{equation}
If $\bm r=\bar{\B}\one$ and $\bm c=\bar{\B}^{\top}\one$, the degree vector is
\begin{equation}
\bm d_{\alpha}=
\begin{bmatrix}
\bar{\B}\bm c+\alpha\bm r\\
\bar{\B}^{\top}\bm r+\alpha\bm c
\end{bmatrix},
\qquad
\Ltw\bm z=\bm d_{\alpha}\odot\bm z-\Atw\bm z.
\label{eq:lmatvec}
\end{equation}
\end{lemma}
\begin{proof}
Multiply the block matrix in Eq.~\eqref{eq:tw} by $[\bm x^{\top},\bm y^{\top}]^{\top}$ and associate the products as $\bar{\B}(\bar{\B}^{\top}\bm x)$ and $\bar{\B}^{\top}(\bar{\B}\bm y)$. Setting $\bm z=\one$ gives the degree vector.
\end{proof}

After precomputing the degree vector, one application of $\Ltw$ uses four sparse products with $\bar{\B}$ or $\bar{\B}^{\top}$. Its arithmetic cost is $O(\operatorname{nnz}(\B))$ and its storage is $O(\operatorname{nnz}(\B)+m+n)$ apart from eigensolver workspace. A direct implementation forms the Gram matrices explicitly; Eq.~\eqref{eq:lmatvec} specifies an exact sparse replacement rather than an approximation.

A Fiedler vector is defined only up to sign, so reversing both axis orders is equivalent. For deterministic output, we orient the vector so that its first coordinate of maximum absolute value is positive, then apply a stable sort. If the active support has several connected components, we order each component independently, arrange components by decreasing total mass, and append zero-marginal rows and columns in their original order. A repeated Fiedler eigenvalue can still make the order nonunique, which is a limitation of spectral seriation rather than a numerical bug.

\subsection{Mosaic Layout}
\label{sec:mosaic}
Let $\A$ be the reordered nonnegative matrix. Let
\begin{equation}
r_i=\sum_j A_{ij},\qquad c_j=\sum_i A_{ij},\qquad T=\sum_{i,j}A_{ij}>0.
\end{equation}
Assume positive active marginals, and define normalized marginals $R_i=r_i/T$ and $C_j=c_j/T$.

\begin{definition}[Mosaic layout]
\label{def:mosaic}
Set cumulative boundaries
\begin{equation}
Y_0=0,\quad Y_i=\sum_{k=1}^{i}R_k,
\qquad
X_0=0,\quad X_j=\sum_{\ell=1}^{j}C_{\ell}.
\label{eq:mosaic-edges}
\end{equation}
In the \Mosaic layout, cell $(i,j)$ occupies
\begin{equation}
\mathcal R_{ij}=[X_{j-1},X_j]\times[Y_{i-1},Y_i],
\label{eq:mosaic-cell}
\end{equation}
and its fill color is a monotone mapping $g(A_{ij})$ of the observed entry.
\end{definition}

\begin{proposition}[Tiling and area semantics]
\label{prop:mosaic}
The rectangles in Definition~\ref{def:mosaic} cover the unit square and have disjoint interiors. Every cell in row $i$ has height $R_i$, every cell in column $j$ has width $C_j$, and
\begin{equation}
\operatorname{area}(\mathcal R_{ij})=R_iC_j=\frac{r_ic_j}{T^2}.
\label{eq:mosaic-area}
\end{equation}
If $\A$ is a contingency table, its expected count under row and column independence is $E_{ij}=r_ic_j/T$, so the cell area equals $E_{ij}/T$.
\end{proposition}
\begin{proof}
The normalized marginals satisfy $\sum_iR_i=\sum_jC_j=1$. Their cumulative sums partition each axis into adjacent intervals. Taking all Cartesian products therefore partitions the unit square. The height, width, and area follow directly from the interval lengths. The independence statement follows from the standard expected count formula.
\end{proof}

This proposition separates the visual channels. Geometry encodes the two marginals and therefore the independence baseline. Color encodes the observed cell value. Cell area does not encode $A_{ij}$ itself. The mathematical layout has no interior gaps. A white region in a rendered figure comes from the color mapping of a zero or small entry, or from optional divider lines. Rare rows or columns can become visually narrow, so \Mosaic is intended for overview and mass context rather than precise cell value comparison.

\begin{definition}[MHeatmap workflow]
\MHeatmap first applies \TW seriation to obtain $(\pi_R,\pi_C)$. It then renders the reordered matrix either as an equal area heatmap or with the \Mosaic layout in Definition~\ref{def:mosaic}.
\end{definition}

\section{Evaluation}
\label{sec:evaluation}

\subsection{Evaluation Questions and Scope}
We evaluate the two stages of \MHeatmap with evidence matched to their claims. The ordering stage is evaluated quantitatively because it produces explicit row and column permutations. The \Mosaic stage is evaluated through controlled visual comparisons because it changes the representation of the same reordered matrix rather than the underlying order.

The evaluation addresses four questions. First, does \TW recover rectangular block and band structure more consistently than representative ordering baselines? Second, does it improve over \OW, which uses the same spectral pipeline but excludes the length-two terms? Third, on which matrix semantics do shared-neighbor affinities help, and where do they provide little benefit? Fourth, when the row and column order is held fixed, what additional information is made visible by \Mosaic geometry?

\subsection{Controlled Synthetic Benchmark}
The synthetic component contains five structural families at three matrix sizes, with 20 independent random seeds for every family and size combination. This gives
\[
5\ \text{families}\times 3\ \text{sizes}\times 20\ \text{seeds}=300
\]
matrices. The sizes are $100\times90$, $200\times180$, and $400\times360$. All families are generated from a common rectangular block template, after which the rows and columns are randomly permuted before any method is applied. The algorithms receive only the permuted weighted matrix and do not receive block labels.

Family A is a clean one-to-one block control. Each principal row group corresponds to one principal column group, so a direct row to column signal is already strong. Family B adds paired subgroup overlap, which weakens one-to-one separation while preserving shared-neighbor structure. Family C introduces shared super-prototypes, so several subgroups inherit related column profiles and form nested organization. Family D adds noise to the shared-prototype setting. Family E introduces cross-block leakage, so relation mass can spill outside the intended diagonal blocks. The sequence therefore separates an easy direct-link regime from regimes in which shared neighborhoods, noise, and leakage become increasingly important.

\subsection{Real Matrix Benchmark}
The real component contains seven nonnegative rectangular matrices with different semantics and aspect ratios. Two official crosswalks, SIC to NAICS and CIP to SOC, test many-to-many mappings between classification systems~\cite{census_naics,nces_cipsoc}. An ACS occupation by industry contingency table tests a small aggregate relation matrix~\cite{census_acs}. LODES home county by work county flows test a matrix whose organization is partly geographic~\cite{census_lodes}. A 20 Newsgroups term by document matrix tests text relations~\cite{uci20newsgroups}. An MBTA route by station matrix tests route incidence and local network structure~\cite{mbta_gtfs}. An OpenAlex author by topic matrix tests specialization across distinct entity types~\cite{priem2022openalex}.

Dataset construction follows fixed external rules rather than a search for submatrices on which \TW performs well. Complete official crosswalks are used when available. Geographic, text, transit, and bibliographic subsets are fixed before ordering. Only structurally empty rows and columns are removed. Table~\ref{tab:benchmark} reports the final matrix dimensions.

\subsection{Compared Methods}
We compare seven orderings that represent complementary design choices.

\paragraph{Original (O).}
This baseline keeps the supplied row and column orders. It measures how much structure is visible without algorithmic intervention.

\paragraph{Marginal sorting (M).}
Rows and columns are sorted independently by decreasing row and column sums. This exposes degree or frequency imbalance but does not use relation profiles.

\paragraph{HC with optimal leaf ordering (HC).}
Hierarchical clustering is applied separately to row profiles and column profiles, followed by optimal leaf ordering~\cite{barjoseph2001optimal}. It represents profile-based clustering that does not construct a joint row and column spectral coordinate.

\paragraph{One-Walk (OW).}
\OW applies the same Laplacian and Fiedler ordering pipeline as \TW to the ordinary bipartite lift in Eq.~\eqref{eq:lift}. It excludes $\bar{\B}\bar{\B}^{\top}$ and $\bar{\B}^{\top}\bar{\B}$. This is the central ablation because the difference between OW and TW is exactly the length-two augmentation.

\paragraph{Correspondence analysis SVD (CA).}
Correspondence analysis obtains coupled row and column coordinates from a marginally normalized matrix and orders each axis by its first nontrivial coordinate~\cite{greenacre2017correspondence}. It represents a factorization-based alternative.

\paragraph{Alternating median (MD).}
Alternating median updates row positions from the weighted median positions of adjacent columns and then updates columns from adjacent rows~\cite{cakiroglu2009weightedcrossing}. It represents a local refinement method motivated by weighted bipartite layout.

All methods receive the same nonnegative matrix after removal of structurally empty rows and columns. For spectral methods, reversing the complete coordinate vector gives an equivalent ordering. The reported distance based metrics are invariant when both axes are reversed together.

\subsection{Structure-Recovery Metrics}
A rectangular matrix has different numbers of row and column positions, so raw integer ranks cannot be compared directly. For row permutation $\pi_R$ and column permutation $\pi_C$, we map both axes to the same unit interval:
\begin{equation}
\rho_i=\frac{\pi_R(i)-1}{\max(m-1,1)},
\qquad
\kappa_j=\frac{\pi_C(j)-1}{\max(n-1,1)}.
\label{eq:coords}
\end{equation}
Let $T_B=\sum_{i,j}B_{ij}>0$.

\paragraph{Normalized rectangular 2-SUM (R2S).}
We measure weighted diagonal dispersion by
\begin{equation}
\operatorname{R2S}(\pi_R,\pi_C)
=
\frac{1}{T_B}
\sum_{i,j}B_{ij}(\rho_i-\kappa_j)^2.
\label{eq:r2s}
\end{equation}
R2S lies in $[0,1]$. A lower value means that more relation mass is placed near corresponding normalized positions on the two axes. This is the rectangular form of a weighted quadratic arrangement objective. Table~\ref{tab:benchmark} multiplies R2S by $10^2$ for readability.

\paragraph{Band concentration.}
For a normalized half width $\tau$, we define
\begin{equation}
\operatorname{Band}_{\tau}(\pi_R,\pi_C)
=
\frac{1}{T_B}
\sum_{i,j}B_{ij}
\mathbb I\!\left\{\lvert\rho_i-\kappa_j\rvert\le\tau\right\}.
\label{eq:band}
\end{equation}
We report $\tau=0.10$. Higher values indicate that a larger fraction of mass lies in a narrow diagonal band. R2S penalizes distance continuously, whereas Band@10\% applies a fixed visual tolerance. The two measures can therefore disagree when a method concentrates most mass tightly but leaves a small number of distant entries.

\paragraph{MWB-AUC.}
MWB-AUC summarizes block separability across a sweep of diagonal block cuts. For each cut scale, the benchmark computes the fraction of total mass retained within matched contiguous diagonal blocks. The area under this curve is larger when diagonal blocks remain well separated over a broad range of cut scales. Unlike R2S and Band@10\%, MWB-AUC is also used internally to choose $\alpha$ for adaptive \TW.

\subsection{Parameter Selection and Reporting Protocol}
Adaptive \TW evaluates
\[
\alpha\in\{1,2,4,6,8,12\}
\]
for each input matrix and selects the value with the largest internal MWB-AUC. The selection is unsupervised and does not use synthetic block labels or external task labels. Nevertheless, it gives \TW one data-dependent choice that the fixed baselines do not receive. It also makes MWB-AUC a selection-aligned measure. We retain MWB-AUC in the complete table for transparency, but base the main comparative claim on R2S and Band@10\%, neither of which is used to choose $\alpha$.

For each synthetic family and size, Table~\ref{tab:benchmark} reports the arithmetic mean over 20 seeds. Each real matrix is evaluated once at its full reported size. Boldface marks the best method within one row and metric, and underlining marks the runner-up, including ties. The table reports every family and size rather than only an aggregate so that readers can inspect dependence on matrix scale and motif difficulty.

The \Mosaic evaluation uses four Tabula Sapiens cell type by cluster contingency tables~\cite{tabula_sapiens}. For every tissue, the equal-area \TW panel and the \Mosaic panel use identical row and column permutations and the same matrix values. The only change is the geometry from equal row heights and column widths to the marginal proportions in Definition~\ref{def:mosaic}. This design isolates what can be attributed to \Mosaic. The qualitative analysis is intended to establish interpretability of the encoding, not universal perceptual superiority over equal-area heatmaps or marginal bar charts.

\section{Results}
\label{sec:results}
Table~\ref{tab:benchmark} is the complete quantitative result. The synthetic rows retain all 15 family and size regimes, and the real rows retain all seven datasets. The table also keeps MWB-AUC visible so that the effect of the adaptive selector can be audited rather than hidden.

\begin{table*}[t]
\centering
\scriptsize
\setlength{\tabcolsep}{3pt}
\resizebox{\textwidth}{!}{%
\begin{tabular}{ll rrrrrrr rrrrrrr rrrrrrr}
\toprule
& & \multicolumn{7}{c}{$10^2 \times$ R2S $\downarrow$}
& \multicolumn{7}{c}{Band@10\% $\uparrow$}
& \multicolumn{7}{c}{MWB-AUC $\uparrow$} \\
&
& O & M & HC & OW & CA & MD & TW
& O & M & HC & OW & CA & MD & TW
& O & M & HC & OW & CA & MD & TW \\
\midrule
\multicolumn{23}{l}{\emph{Synthetic component}} \\
\midrule

A: Clean one-to-one
& 400$\times$360
& 16.70 & 13.84 & 10.36 & \underline{0.52} & 0.59 & 0.79 & \textbf{0.43}
& 0.20 & 0.21 & 0.24 & \underline{0.86} & 0.85 & 0.76 & \textbf{0.90}
& 0.44 & 0.47 & 0.53 & \underline{0.92} & \underline{0.92} & 0.89 & \textbf{0.93} \\
& 200$\times$180
& 16.88 & 13.72 & 10.67 & \underline{0.55} & 0.63 & 0.89 & \textbf{0.42}
& 0.20 & 0.23 & 0.26 & \underline{0.87} & 0.86 & 0.78 & \textbf{0.92}
& 0.44 & 0.48 & 0.54 & \underline{0.92} & \underline{0.92} & 0.89 & \textbf{0.94} \\
& 100$\times$90
& 17.09 & 13.25 & 9.14 & \underline{0.60} & 0.79 & 0.96 & \textbf{0.51}
& 0.21 & 0.24 & 0.34 & \underline{0.89} & 0.85 & 0.88 & \textbf{0.92}
& 0.44 & 0.49 & 0.60 & \underline{0.93} & 0.91 & 0.91 & \textbf{0.94} \\

\midrule
B: Paired subgroup overlap
& 400$\times$360
& 16.72 & 14.04 & 8.10 & \underline{1.04} & 1.11 & 1.08 & \textbf{0.88}
& 0.19 & 0.21 & 0.31 & 0.68 & 0.67 & \underline{0.69} & \textbf{0.74}
& 0.44 & 0.47 & 0.60 & \underline{0.87} & \underline{0.87} & \underline{0.87} & \textbf{0.89} \\
& 200$\times$180
& 16.76 & 14.15 & 8.53 & \underline{1.15} & 1.20 & 1.34 & \textbf{0.82}
& 0.20 & 0.22 & 0.27 & 0.67 & 0.70 & \underline{0.72} & \textbf{0.78}
& 0.44 & 0.47 & 0.56 & \underline{0.87} & \underline{0.87} & \underline{0.87} & \textbf{0.90} \\
& 100$\times$90
& 16.90 & 13.68 & 10.39 & \underline{1.03} & 1.32 & 1.93 & \textbf{0.94}
& 0.22 & 0.24 & 0.32 & \underline{0.76} & 0.72 & 0.73 & \textbf{0.79}
& 0.44 & 0.49 & 0.55 & \underline{0.89} & 0.87 & 0.86 & \textbf{0.90} \\

\midrule
C: Shared super-prototype
& 400$\times$360
& 16.92 & 14.07 & 8.71 & \underline{1.13} & 1.35 & 1.78 & \textbf{0.93}
& 0.19 & 0.22 & 0.27 & 0.67 & 0.64 & \underline{0.69} & \textbf{0.73}
& 0.43 & 0.47 & 0.58 & \underline{0.87} & 0.85 & 0.85 & \textbf{0.88} \\
& 200$\times$180
& 16.88 & 13.62 & 9.15 & \underline{1.01} & 1.24 & 2.07 & \textbf{0.86}
& 0.20 & 0.23 & 0.31 & \underline{0.72} & 0.69 & \underline{0.72} & \textbf{0.77}
& 0.44 & 0.49 & 0.58 & \underline{0.88} & 0.87 & 0.85 & \textbf{0.90} \\
& 100$\times$90
& 17.24 & 13.12 & 7.44 & \underline{1.01} & 1.39 & 3.23 & \textbf{0.86}
& 0.21 & 0.25 & 0.37 & \underline{0.77} & 0.72 & 0.73 & \textbf{0.83}
& 0.44 & 0.50 & 0.64 & \underline{0.89} & 0.87 & 0.83 & \textbf{0.90} \\

\midrule
D: Shared prototype with noise
& 400$\times$360
& 16.63 & 15.12 & 9.85 & 1.28 & 1.54 & \underline{1.18} & \textbf{1.00}
& 0.20 & 0.22 & 0.24 & 0.63 & 0.61 & \underline{0.64} & \textbf{0.70}
& 0.44 & 0.47 & 0.52 & 0.85 & 0.84 & \underline{0.86} & \textbf{0.88} \\
& 200$\times$180
& 16.74 & 14.35 & 8.94 & \underline{1.23} & 1.67 & 1.51 & \textbf{1.04}
& 0.20 & 0.23 & 0.29 & \underline{0.69} & 0.62 & 0.65 & \textbf{0.72}
& 0.44 & 0.48 & 0.57 & \underline{0.87} & 0.84 & 0.85 & \textbf{0.88} \\
& 100$\times$90
& 17.07 & 13.51 & 7.70 & \underline{1.14} & 1.65 & 2.67 & \textbf{0.94}
& 0.21 & 0.25 & 0.40 & \underline{0.74} & 0.67 & 0.72 & \textbf{0.79}
& 0.44 & 0.50 & 0.64 & \underline{0.88} & 0.85 & 0.84 & \textbf{0.90} \\

\midrule
E: Cross-block leakage
& 400$\times$360
& 16.65 & 14.70 & 9.46 & \underline{2.09} & \underline{2.09} & 2.65 & \textbf{1.87}
& 0.19 & 0.22 & 0.26 & 0.55 & \underline{0.57} & \textbf{0.60} & 0.56
& 0.44 & 0.47 & 0.56 & \underline{0.81} & \underline{0.81} & \textbf{0.82} & \textbf{0.82} \\
& 200$\times$180
& 16.88 & 13.88 & 7.59 & \underline{1.86} & 2.16 & 3.28 & \textbf{1.65}
& 0.20 & 0.23 & 0.38 & 0.60 & 0.59 & \underline{0.61} & \textbf{0.63}
& 0.44 & 0.49 & 0.63 & \underline{0.83} & 0.82 & 0.80 & \textbf{0.84} \\
& 100$\times$90
& 16.87 & 13.15 & 8.36 & \underline{1.61} & 2.12 & 4.25 & \textbf{1.60}
& 0.22 & 0.25 & 0.33 & \underline{0.67} & 0.64 & \underline{0.67} & \textbf{0.70}
& 0.44 & 0.50 & 0.61 & \underline{0.85} & 0.83 & 0.79 & \textbf{0.86} \\

\midrule
\multicolumn{23}{l}{\emph{Real-world component}} \\
\midrule

SIC $\rightarrow$ NAICS
& 1004$\times$1176
& 2.61 & 10.78 & 2.11 & \underline{0.68} & 2.88 & 0.76 & \textbf{0.18}
& 0.66 & 0.30 & 0.52 & \textbf{0.98} & 0.51 & 0.90 & \underline{0.96}
& 0.82 & 0.56 & 0.78 & \underline{0.94} & 0.73 & 0.92 & \textbf{0.97} \\

CIP $\rightarrow$ SOC
& 2143$\times$868
& 15.90 & 16.04 & 12.18 & \underline{2.52} & 9.77 & 12.37 & \textbf{1.67}
& 0.19 & 0.27 & 0.20 & \underline{0.43} & 0.30 & 0.17 & \textbf{0.61}
& 0.44 & 0.50 & 0.53 & \underline{0.77} & 0.57 & 0.40 & \textbf{0.83} \\

ACS OCCP $\times$ INDP
& 24$\times$10
& 20.67 & 8.87 & 15.94 & 3.39 & \textbf{2.23} & \underline{3.05} & 3.61
& 0.26 & 0.52 & 0.32 & \underline{0.88} & \textbf{0.91} & 0.85 & \underline{0.88}
& 0.52 & 0.74 & 0.53 & \underline{0.92} & \textbf{0.95} & 0.91 & \underline{0.92} \\

LODES Home $\times$ Work County
& 95$\times$95
& 10.98 & 3.98 & 2.78 & \textbf{1.89} & 5.42 & 4.54 & \underline{1.90}
& 0.60 & 0.72 & 0.69 & \underline{0.78} & \textbf{0.79} & 0.77 & \underline{0.78}
& 0.69 & 0.82 & 0.81 & \textbf{0.87} & \underline{0.86} & 0.83 & \textbf{0.87} \\

20 Newsgroups Term $\times$ Document
& 60$\times$48
& 6.98 & 0.99 & 2.75 & \textbf{0.43} & 0.56 & \underline{0.47} & \textbf{0.43}
& 0.27 & 0.92 & 0.07 & \underline{0.94} & 0.93 & \underline{0.94} & \textbf{0.95}
& 0.68 & 0.93 & 0.72 & 0.94 & \textbf{0.97} & \underline{0.95} & \underline{0.95} \\

MBTA GTFS Route $\times$ Station
& 8$\times$125
& 17.96 & 9.39 & 3.33 & \textbf{1.25} & 3.09 & \underline{2.19} & \textbf{1.25}
& 0.02 & 0.01 & 0.09 & 0.10 & \underline{0.14} & \textbf{0.19} & 0.11
& 0.04 & 0.02 & 0.13 & 0.14 & \underline{0.19} & \textbf{0.27} & 0.15 \\

OpenAlex Author $\times$ Topic
& 60$\times$8
& 19.72 & 12.88 & 11.79 & \underline{0.98} & 2.53 & 3.36 & \textbf{0.92}
& 0.39 & 0.56 & 0.51 & \underline{0.99} & 0.88 & 0.98 & \textbf{1.00}
& 0.56 & 0.68 & 0.68 & \underline{0.99} & 0.94 & \underline{0.99} & \textbf{1.00} \\

\bottomrule
\end{tabular}%
}

\caption{Combined rectangular evaluation benchmark. Best entry per row and metric is in \textbf{bold}, runner-up is \underline{underlined}. Synthetic component (top): mean over 20 seeds per family–size regime. Real-world component (bottom): datasets at full size. MWB-AUC is the internal \(\alpha\)-selection metric for TW; R2S and Band@10\% are not used for \(\alpha\) selection.}

\label{tab:benchmark}
\end{table*}

\subsection{Synthetic Structure Recovery}
Across the synthetic component, \TW obtains the lowest R2S in all 15 family and size regimes. It obtains the highest Band@10\% in 14 of 15 regimes and the highest or tied highest MWB-AUC in all 15. The MWB-AUC result is expected to favor \TW because the same score selects $\alpha$. More importantly, the improvement also appears on the two metrics that are not used for selection.

Averaged over the 15 regimes, the reported $10^2\times$R2S decreases from $1.150$ for \OW to $0.983$ for \TW, a relative reduction of approximately $14.5\%$. Mean Band@10\% increases from $0.718$ to $0.765$, an absolute increase of $0.047$. Since \OW and \TW share the same normalization, Laplacian construction, eigensolver concept, coordinate split, and sorting rule, this comparison isolates the contribution of the length-two terms.

\paragraph{Family A, clean one-to-one blocks.}
Direct row to column relations already give \OW a strong signal. \TW still improves R2S at all three sizes, from $0.52$, $0.55$, and $0.60$ to $0.43$, $0.42$, and $0.51$, respectively. Band@10\% also rises in every size. These modest gains are consistent with the method design: shared-neighbor affinities reinforce an already separable pattern but are not essential to recover it.

\paragraph{Families B and C, overlapping and nested structure.}
The benefit is clearer when entities on the same side share relation profiles. In Family B, \TW improves R2S over \OW in all sizes, with the largest change at $200\times180$, from $1.15$ to $0.82$. Band@10\% rises from $0.67$ to $0.78$ in the same regime. Family C shows the same pattern across shared super-prototypes. The results support the mechanism in Theorem~\ref{thm:energy}: the Gram terms add direct smoothness between rows with shared columns and between columns with shared rows.

\paragraph{Family D, shared prototypes with noise.}
Noise weakens individual entries but does not remove the neighborhood pattern. \TW remains best on both independent metrics at all three sizes. For the largest matrix, R2S decreases from $1.28$ for \OW to $1.00$, and Band@10\% increases from $0.63$ to $0.70$. This suggests that aggregating two-step evidence can preserve block organization when single edges are less reliable.

\paragraph{Family E, cross-block leakage.}
This family gives the most restrictive synthetic case. \TW retains the lowest R2S at all sizes, but it is not best on Band@10\% for the $400\times360$ matrix. MD places $0.60$ of the mass within the fixed band, compared with $0.56$ for \TW, while \TW has a substantially lower R2S than MD, $1.87$ versus $2.65$. The disagreement is informative. MD puts slightly more mass inside the chosen band threshold, whereas \TW reduces the overall squared distance of relation mass from the diagonal. No single metric completely describes the visual organization of a reordered matrix.

\subsection{Real Matrix Results}
The strongest improvements occur in relational matrices where same-side entities are naturally compared through shared neighbors.

\paragraph{Classification crosswalks.}
For SIC to NAICS, \TW reduces $10^2\times$R2S from $0.68$ for \OW to $0.18$, a reduction of approximately $73.5\%$. \OW has the slightly higher Band@10\%, $0.98$ versus $0.96$, which indicates that both methods form a narrow band but \TW reduces the remaining long-distance mass more strongly. For CIP to SOC, \TW improves both independent metrics. R2S decreases from $2.52$ to $1.67$, and Band@10\% increases from $0.43$ to $0.61$. Crosswalks are a favorable setting because two source codes may be related precisely when they map to similar target codes, and the reverse relation also holds for target codes.

\paragraph{Text and bibliographic matrices.}
On 20 Newsgroups, \TW ties \OW for the best R2S at $0.43$ and raises Band@10\% from $0.94$ to $0.95$. On OpenAlex, R2S improves from $0.98$ to $0.92$, and Band@10\% reaches $1.00$. The OpenAlex result is visually consistent with Figure~\ref{fig:teaser}, where authors with related topic profiles and topics with related author profiles form coordinated blocks.

\paragraph{Small and spatially organized matrices.}
The ACS matrix is only $24\times10$. CA gives the best R2S and Band@10\%, while \TW is not competitive on R2S. In this regime, a low-dimensional factor coordinate is sufficient and the additional shared-neighbor graph does not add useful separation. LODES is governed partly by geographic adjacency. \OW and \TW are nearly identical on R2S, and CA has the highest Band@10\%. MBTA is the clearest boundary case. \TW ties \OW on R2S, but MD has the best Band@10\% and MWB-AUC. Route and station order is influenced by local transit topology, which is not equivalent to similarity through shared bipartite neighborhoods.

These results support a scoped conclusion. \TW is useful when latent organization is relational and shared-neighbor profiles carry information beyond direct edges. It is not a universal substitute for geometric, spatial, or very low-dimensional ordering criteria.

\subsection{Baseline Behavior Across Regimes}
The non-spectral baselines clarify what type of information is needed for the benchmark. Original and marginal orders remain far from the diagonal on every synthetic family. Marginal sorting can expose large rows and columns, but it cannot recover a correspondence between groups on the two axes. HC reduces R2S substantially relative to these two controls, which confirms that profile similarity is useful, but its independently constructed row and column trees do not match the joint spectral methods on the relational motifs.

CA is most competitive when a small number of factor coordinates already describes the matrix. This is visible in Family A and in the ACS table. MD is most competitive when a local neighborhood update agrees with the target structure, including the large leakage regime and the transit example. Its weaker performance on several nested synthetic regimes also shows that local median updates do not consistently replace the global shared-neighbor signal. The diversity of these outcomes is important because it rules out the simpler explanation that any reasonable reordering would obtain the same result.

\subsection{Consistency Across Matrix Size}
The synthetic matrix dimensions increase from $100\times90$ to $400\times360$, which multiplies the number of cells by sixteen. \TW retains the best R2S in every family at every size. It also retains the best Band@10\% except for the largest leakage matrix. The absolute values do not vary monotonically with size because the generator changes the number and realization of weighted relations with the seed. The stable win pattern is the relevant observation: the advantage is not restricted to one selected matrix dimension or one aspect ratio.

The reported synthetic entries are means across matched random instances. The table does not use a hypothesis test to turn these deterministic ordering outputs into binary significance claims. The per-seed records in the anonymous artifact permit paired analysis because every method is evaluated on the same generated matrix for each seed.

\subsection{What the Three Metrics Reveal}
The complete table shows why three complementary summaries are useful. R2S measures all weighted distances to the diagonal. Band@10\% asks whether mass falls inside one fixed visual tolerance. MWB-AUC rewards persistence of contiguous diagonal blocks across its cut sweep. A method can therefore improve one property without improving all three.

SIC to NAICS is one example. \TW has much lower R2S than \OW, while \OW has slightly higher Band@10\%. Family E at $400\times360$ is another example, where MD has the highest Band@10\% but \TW has the lowest R2S. We therefore avoid reducing the benchmark to one aggregate rank. The primary evidence for the adaptive method is its consistency across R2S and Band@10\%, which are not used for parameter selection.

\subsection{Mosaic Case Analysis}
Figure~\ref{fig:cases} shows four Tabula Sapiens tissues. The transition from the first row to the second changes only the order. Dominant cell type and cluster correspondences move toward the diagonal, making split, merge, and leakage patterns easier to locate. The transition from the second row to the third keeps that \TW order fixed and changes only the geometry.

In the \Mosaic row, a frequent cell type receives more vertical space because its row marginal is larger. A frequent computational cluster receives more horizontal space because its column marginal is larger. A large rectangle therefore indicates a large independence-baseline area $r_ic_j/T^2$, while a dark fill indicates a large observed entry $A_{ij}$. This distinction is important. The layout places a split or merge relation in the context of the total sizes of both participating categories, but rectangle area is not itself the observed count.

Across the four tissues, rare categories contract and high-mass categories expand. This makes it easier to distinguish a visually similar split involving a major cell population from one involving a small population. It also creates a limitation: very small marginals can become difficult to inspect, and unequal geometry weakens exact cell-to-cell comparison. The figure therefore supports the intended overview task but does not establish a general perceptual advantage for every matrix reading task.

\begin{figure*}[t]
\centering
\includegraphics[width=\textwidth]{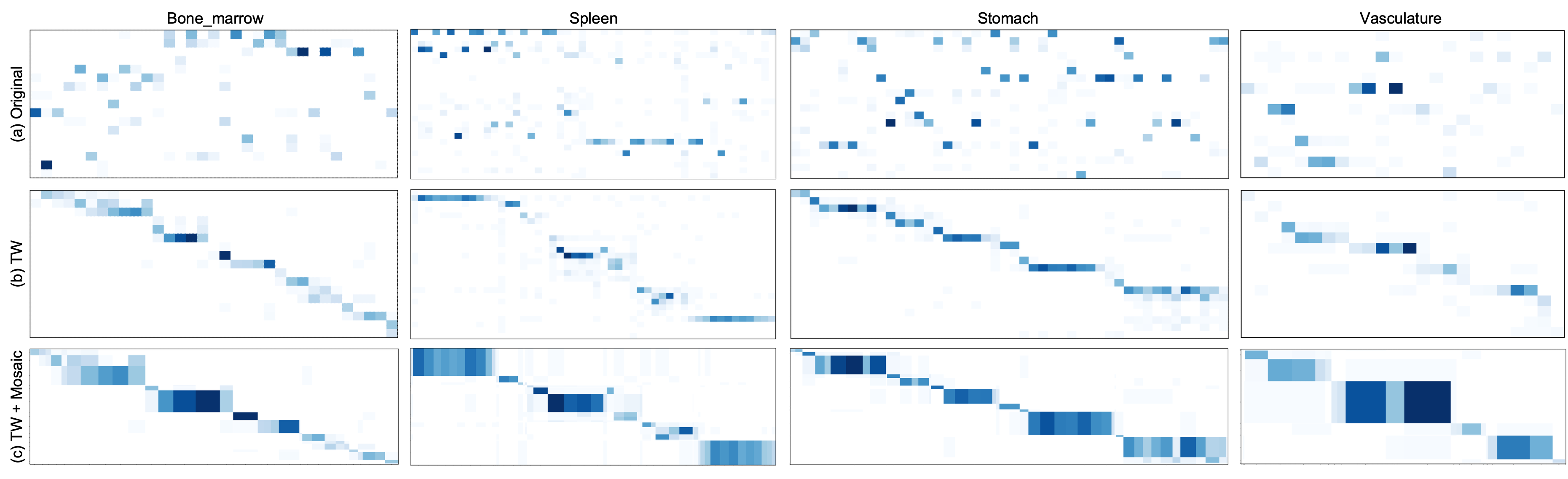}
\caption{Four cell type by cluster contingency matrices from Tabula Sapiens. The first row uses the input order with equal-area cells. The second row applies \TW and keeps equal-area cells. The third row keeps exactly the same \TW row and column orders and applies \Mosaic geometry. The second-to-third-row comparison therefore isolates the layout effect.}
\Description{Four matrices are displayed in three rows. The top row is scattered. The middle row has diagonal blocks with equal-size cells. The bottom row preserves those blocks while assigning unequal row heights and column widths.}
\label{fig:cases}
\end{figure*}

\section{Discussion and Conclusion}
\label{sec:discussion}
\MHeatmap combines a spectral seriation method with a marginal proportional matrix view. The main algorithmic contribution is \TW. Its adjacency retains observed row to column edges and adds the two same side affinities induced by length two walks. Theorem~\ref{thm:energy} shows that the resulting Fiedler coordinate balances smoothness among rows with shared columns, smoothness among columns with shared rows, and alignment of directly related row and column vertices.

\paragraph{Scope of the ordering claim.}
The complete benchmark provides both positive evidence and boundary conditions. \TW is consistently strong when shared profiles define overlapping or nested groups, and it performs well on crosswalk and bibliographic matrices. Its benefit is smaller when direct one step structure is already sufficient, when a matrix is very small, or when spatial and route locality govern the desired order. The two walk terms reward shared bipartite neighborhoods. They do not directly encode geographic distance, temporal order, or route continuity unless those concepts are reflected in the neighborhoods.

The three metrics also define a particular notion of success. They favor arrangements in which relation mass can be organized near matched positions on the normalized axes. This is appropriate for correspondences and block or band motifs, but it is not the only meaningful matrix objective. Hub, checkerboard, and spatial patterns need not become narrowly diagonal. The LODES and MBTA results illustrate this boundary and prevent the benchmark from implying that every matrix should adopt one visual form.

\paragraph{Parameter and computational qualifications.}
Adaptive selection of $\alpha$ gives \TW more flexibility than a fixed method. MWB-AUC is therefore identified as selection aligned, while the principal comparison uses R2S and Band@10\%, which do not choose $\alpha$. The table evaluates the implemented adaptive procedure rather than a parameter free estimator. A fixed global $\alpha$ or a sensitivity curve would be an additional robustness check, but the present conclusion is restricted to the adaptive method and the independent outcome metrics.

Lemma~\ref{lem:matvec} proves that dense Gram matrices are not required and that one matrix vector product is linear in $\operatorname{nnz}(B)$ apart from vector operations. It does not establish wall clock scalability because eigensolver convergence also depends on the eigengap, conditioning, and tolerance. The current experiments evaluate ordering quality, so we make no empirical runtime or memory superiority claim.

\paragraph{Spectral and degree effects.}
A Fiedler order can be nonunique when the relevant eigenvalue is repeated or coordinate values tie. Deterministic component, sign, and tie handling are therefore part of the algorithm and must be recorded with the submitted permutations. In addition, the Gram terms aggregate weighted shared neighbors and can emphasize high mass vertices. Maximum entry normalization removes arbitrary global units but does not remove degree heterogeneity. A degree normalized variant may be preferable when profile shape matters more than shared mass, but it is outside the present method.

\paragraph{Role of Mosaic.}
The \Mosaic stage is defined separately from ordering. Its complete tiling makes row and column marginals visible inside the matrix body, while color retains observed values. It is useful when marginal context affects the interpretation of a split, merge, or block, but it can compress rare categories and does not replace an equal area heatmap for precise comparison. The paired panels hold the order fixed so that the layout effect is not confused with the ordering effect. They establish interpretability, not universal perceptual superiority. Such a claim would require a task specific study against equal area heatmaps, marginal bars, and classical mosaic displays.

In summary, \TW extends bipartite spectral seriation with explicit shared neighbor structure and produces coordinated orders for weighted rectangular matrices. The full table shows when this structure improves recovery and where alternative assumptions remain competitive. \Mosaic provides a precisely defined proportional rendering when row and column mass is part of the analysis.

\bibliographystyle{ACM-Reference-Format}
\bibliography{references}

\appendix

\section{Additional Qualitative Benchmark Cases}
\label{app:additional-cases}
Figure~\ref{fig:appendix-cases} provides two additional full-workflow examples. The top row uses the SIC to NAICS crosswalk, and the bottom row uses the OpenAlex author by topic matrix. Each example repeats the same controlled sequence used in the main paper: input order with equal-area cells, \TW with equal-area cells, and the same \TW order with \Mosaic geometry. These examples are supplementary to the quantitative results in Table~\ref{tab:benchmark}.

\begin{figure*}[t]
\centering
\includegraphics[width=\textwidth]{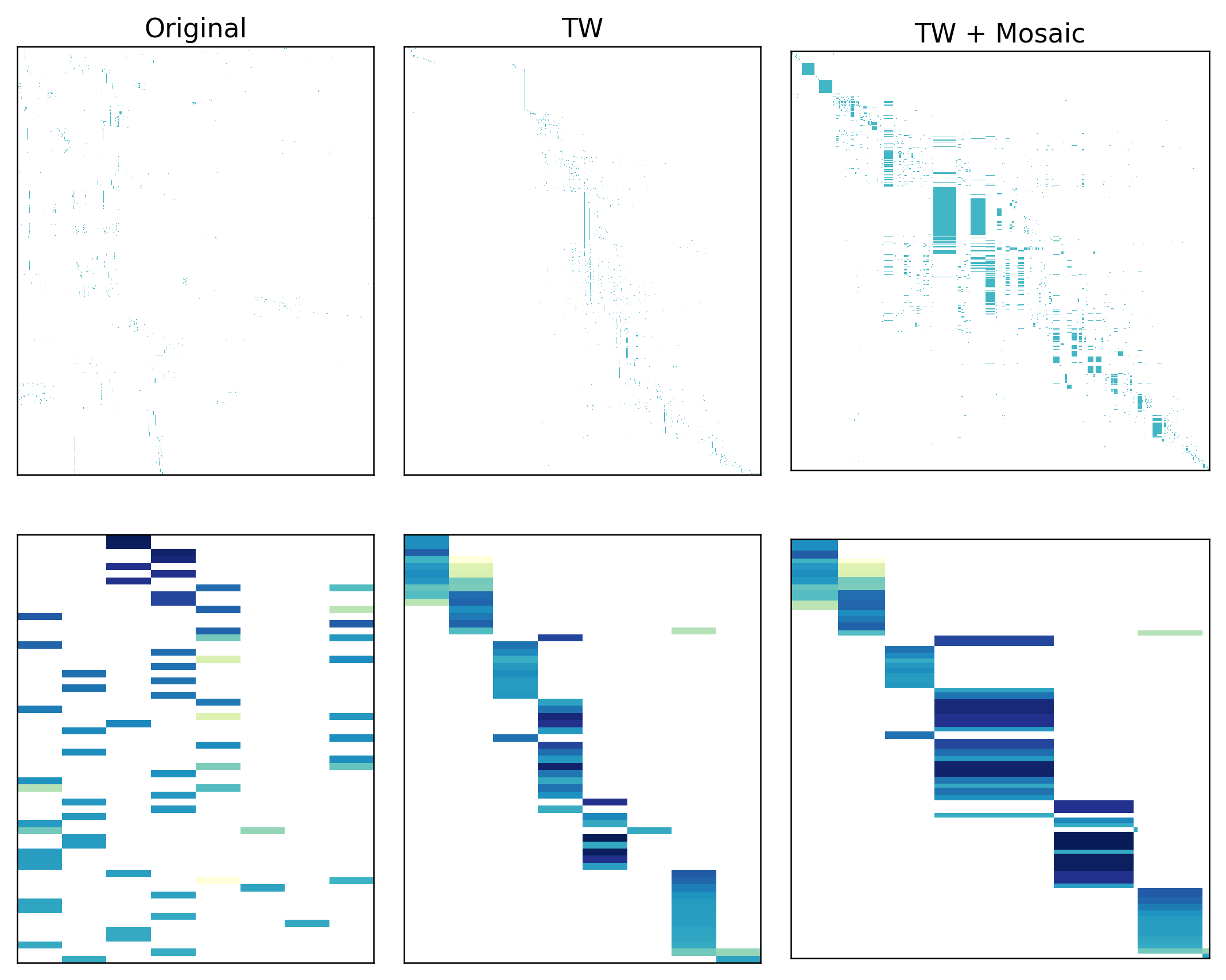}
\caption{Additional qualitative examples. Top: SIC to NAICS. Bottom: OpenAlex author by topic. Columns show the input order, \TW with equal-area cells, and the same \TW order with \Mosaic geometry.}
\Description{Two matrix examples are each shown in three columns. The first column is scattered, the second has diagonal organization after Two-Walk ordering, and the third retains the order while changing row heights and column widths.}
\label{fig:appendix-cases}
\end{figure*}

\section{Reproducibility Materials}
Our code is at \url{https://anonymous.4open.science/r/MHeatmap-D6FD/}. It contains the benchmark implementation, fixed data construction scripts or download instructions, synthetic generators, all baseline wrappers, per-instance permutations and metric records, selected values of $\alpha$, and scripts that regenerate the complete table and every submitted figure.

\end{document}